\documentclass[runningheads]{llncs}
\usepackage{graphicx}
\usepackage{amsmath}
\usepackage{amssymb}
\usepackage{bbold}
\usepackage{eqnarray}
\usepackage{bbm}
\usepackage{xcolor}
\usepackage{url}
\usepackage[chapter]{algorithm}
\usepackage{algorithmic}

\newcommand{\ONE}{\mathbb{1}}

\newcommand{\EE}{\mathbb{E}}
\newcommand{\RR}{\mathbb{R}}
\newcommand{\BB}{\mathbb{B}}
\newcommand{\Tau}{\mathbb{T}}

\def\ShowNotes{0}
\newcommand{\blue}[1]{\ifnum\ShowNotes=1{\color{blue} #1} \fi}
\newcommand{\red}[1]{\ifnum\ShowNotes=1{\color{red} #1} \fi}
\newcommand{\green}[1]{\ifnum\ShowNotes=1{\color{green} #1} \fi}

\DeclareMathOperator*{\argmax}{arg\,max}

\newcommand{\strictorder}{\prec}

\begin{document}
\title{Safe Policy Improvement \\ in Constrained Markov Decision Processes}
%
%
\author{
Luigi Berducci\inst{1} \and
Radu Grosu\inst{1}
}
\authorrunning{L. Berducci et al.}
%
\institute{CPS Group, TU Wien, Austria}
\maketitle              

\begin{abstract}

The automatic synthesis of a policy through reinforcement learning (RL) from a given set of formal requirements depends on the construction of a reward signal and consists of the iterative application of many policy-improvement steps. 
The synthesis algorithm has to balance target, safety, and comfort requirements in a single objective and to guarantee that the policy improvement does not increase the number of safety-requirements violations, especially for safety-critical applications. 
In this work, we present a solution to the synthesis problem by solving its two main challenges: reward-shaping from a set of formal requirements and safe policy update.
For the first, we propose an automatic reward-shaping procedure, defining a scalar reward signal compliant with the task specification. For the second, we introduce an algorithm ensuring that the policy is improved in a safe fashion, with high-confidence guarantees.
We also discuss the adoption of a model-based RL algorithm to efficiently use the collected data and train a model-free agent on the predicted trajectories, where the safety violation does not have the same impact as in the real world.
Finally, we demonstrate in standard control benchmarks that the resulting learning procedure is effective and robust even under heavy perturbations of the hyperparameters.

\end{abstract}


\section{Introduction}
Reinforcement Learning (RL) has become a practical approach for solving complex control tasks in increasingly challenging environments. However, despite the availability of a large set of standard benchmarks with well-defined structure and reward signals, solving new problems remains an art.

There are two major challenges in applying RL to new synthesis problems. The first, arises from the need to define a good reward signal for the problem. We illustrate this challenge with an autonomous-driving (AD) application. In AD, we have numerous requirements that have to be mapped into a single scalar reward signal. In realistic applications, more than 200 requirements need to be considered when assessing the course of action \cite{DBLP:conf/icra/CensiSWYPFF19}.
Moreover, determining the relative importance of the different requirements
is a highly non-trivial task.
In this realm, there are a plethora of regulations, ranging from safety and traffic rules to performance, comfort, legal, and ethical requirements.

The second challenge concerns the RL algorithm with whom we intend to search for an optimal policy.
Considering most of the modern RL algorithms, 
especially model-free policy-gradient methods 
\cite{schulman2017_ppo}, 
they require an iterative interaction with the environment, 
from which they collect fresh experiences for further improving the policy.
Despite their effectiveness, a slight change in the policy parameters 
could have unexpected consequences in the resulting performance.
This fact often results in learning curves with unstable performances,
strongly depending on the tuning of hyperparameters and algorithmic design choices \cite{henderson2018_drl_that_matters}.
The lack of guarantees and robustness in the policy improvement still
limits the application of RL outside controlled environments or simulators.

In this paper, we tackle the two problems mentioned above by proposing a complete design pipeline for safe policies: from the definition of the problem to the safe optimization of the policy (Fig. \ref{fig:workflow}). 
We discuss how to structure the problem from a set of formal requirements and enrich the reward signal while keeping a sound formulation.
Moreover, we define the conditions which characterize a correct-by-construction algorithm in this context and demonstrate its realizability with high-confidence off-policy evaluation.
We propose two algorithms, one
model-free and one model-based, respectively.
We formally prove that they are correct by construction and 
evaluate their performance empirically.

\begin{figure}
    \centering
    \includegraphics[width=\textwidth]{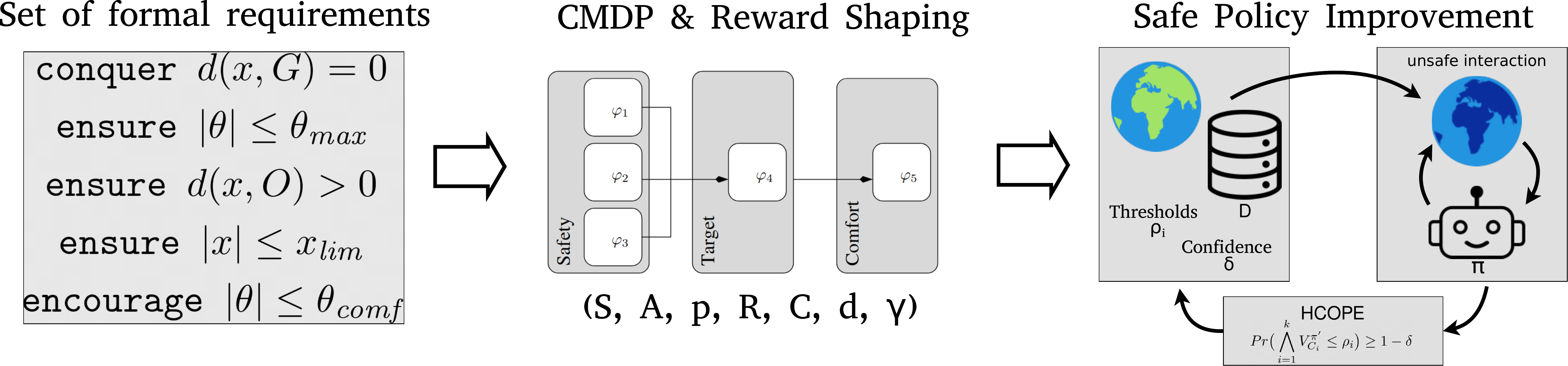}
    \caption{Overall RL pipeline to solve formally-specified control tasks.}
    \label{fig:workflow}
\end{figure}

\section{Motivating Example}
We motivate our work with a \emph{cart-pole example} extended with safety, target, and comfort requirements, as follows: A \emph{pole} is attached to a \emph{cart} that moves between a \emph{left} and a \emph{right} \emph{limit}, within a flat and frictionless environment. The environment has a \emph{target} area within the limits and a static \emph{obstacle} hanging above the track. We define five requirements for the cart-pole, as shown in Table~\ref{tab:reqs}.

\begin{table}[ht]
    \centering
    \begin{tabular}{|l|l|}
    \hline
    Req ID & Description \\
    \hline \hline
    $\text{Req}_1$     &  The cart shall reach the target in bounded time \\
    $\text{Req}_2$     &  The pole shall never fall from the cart \\
    $\text{Req}_3$     &  The pole shall never collide with the obstacle \\
    $\text{Req}_4$     &  The cart shall never leave the left/right limits \\
    $\text{Req}_5$     &  The cart shall keep the pole balanced within a \\
                       &  comfortably small angle as often as possible \\
    \hline
    \end{tabular}
    \caption{Cart-pole example -- informal requirements}
    \label{tab:reqs}
\end{table}

We aim to teach the cart-pole to satisfy all requirements. The system is controlled by applying a continuous \textit{force} to the cart, allowing the left and right movements of the cart-pole with different velocities. In order to reach the goal and satisfy the target requirement $\text{Req}_1$, the cart-pole must do an uncomfortable and potentially unsafe maneuver: since moving a perfectly balanced pole would result in a collision with the obstacle, thus violating the safety requirement $\text{Req}_3$, the cart-pole must lose balancing and pass below it.
Furthermore, if the obstacle is too large, or the cart does not apply enough force once passing the obstacle, it may not be able to reach the target without falling, thus violating the safety requirement $\text{Req}_2$. 
A sequence of pictures showing the cart-pole successfully overcoming the obstacle in the environment is depicted in Figure~\ref{fig:cartpole_storyboard}.

\begin{figure}
    \includegraphics[width=\textwidth]{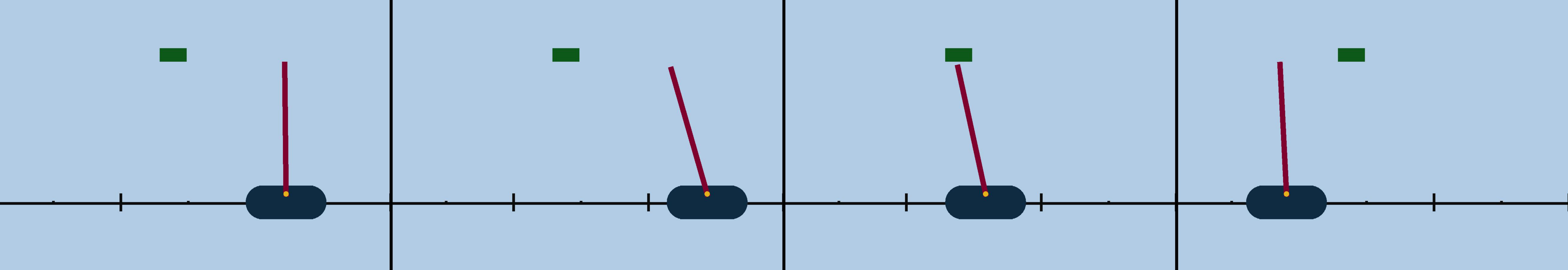}
    \caption{A cart-pole overcomes a hanging obstacle.} 
    \label{fig:cartpole_storyboard}
\end{figure}

Observe that not all requirements have the same importance for this task. We regard the safety requirements as fundamental constraints on the agent's behavior. A safety violation compromises the validity of the entire episode. The target requirement expresses the main objective, and its completion is the agent's reason to be.
Finally, comfort requirements are secondary objectives that should be optimized as long as they do not interfere with the other two classes of requirements.
In the remainder of this paper, we will use this motivating example to illustrate the steps that lead to the proposed methodology.

\section{Related Work}
Safety is a well-known and highly-researched problem in the RL community, and the related literature is wide \cite{DBLP:journals/corr/Brunke2021_SafeLearningInRobotics,DBLP:journals/jmlr/Garcia2015_SurveyInSafeRL}. 
Many different definitions and approaches have emerged over the years, so we will first clarify the interpretations of safety that we adopt.
Much work considers safety as the property of the trained policy, determining whether the agent satisfies a set of constraints. They either try to converge to a safe policy at the end of the training
\cite{Altman1998_CMDP_LagrangianApproach,bertsekas2014_ConstrainedOptimizationAndLagrangeMultiplier,DBLP:journals/corr/Chow2015_RiskConstrainedRL,DBLP:conf/icml/Achiam2017_ConstrinedPolicyOptimization},
or try to ensure safety during the entire training process
\cite{DBLP:journals/corr/Saunders2017_TrialWithoutErrorTowardsSRLViaHumanIntervention,DBLP:conf/nips/Christiano2017_DeepRLFromHumanPreferences,DBLP:journals/corr/Phan2019_NeuralSimplexArchitecture,DBLP:journals/corr/Alshiekh2017_SafeRLViaShielding,DBLP:journals/corr/Dalal2018_SafeExplorationInContinuousActionSpaces,DBLP:journals/corr/ShalevShwartz2016_SafeMultiAgentRLForAD,DBLP:journals/corr/Wilcox2021_LS3LatentSpaceSafeSet,DBLP:journals/ral/Thananjeyan2021_RecoveryRL}.

Placing ourselves in a design perspective, we consider safety as a property of the RL algorithm instead. By safety, we mean that a safe algorithm will not return a policy
with performance below a given threshold, with high probability guarantees.
Building on this interpretation, we will later define the conditions for correct-by-construction RL algorithms. 
In the following, we review the main approaches relevant to our contribution.

\subsubsection{Safe policy improvement}
Early representative of these algorithms
are CPI \cite{kakade2002_cpi,pirotta2013_spi} 
that provide monotonically improving updates
at the cost of enormous sample complexity.
More recently, \cite{DBLP:journals/corr/Schulman2015_TRPO} 
introduced TRPO, 
the first policy-gradient algorithm which 
builds on trust-region methods to 
guarantee policy improvement.
However, its sound formulation must be relaxed
in practice for computational tractability.
The class of algorithms more relevant for our work is based on off-policy evaluation \cite{precup2000_ope,thomas2015_hcpe}.
Seminal works in this field is HCPI \cite{pmlr2015_thomas_hcpi,thomas2015safe}
that use \textit{concentration inequalities} \cite{massart2007_concentration_inequalities} 
to define high-confidence lower bounds on the
\textit{importance sampling} estimates \cite{precup2000_ope}
of the policy performance.
More recently, this class of algorithms has been shown to be part of the general Seldonian framework \cite{thomas2019_nature_paper}. 
We build on this formalism to define our approach and 
differ from the existing work by proposing 
a novel interface with a set of formally specified requirements,
and propose Daedalus-style solutions \cite{pmlr2015_thomas_hcpi} to tackle the problem of iterative improvement 
in an online setting.

\subsubsection{Policy evaluation with statistical model checking}

Statistical guarantees on the policy evaluation 
for some temporal specification
is commonly solved with statistical model checking (SMC) \cite{agha2018_SMC_survey,legay2019_SMC}.
Recent works have proposed SMC for evaluating a NN decision policy operating in an MDP \cite{gros2020deepSMC}.
However, while SMC relies on collecting a large number of executions of the policy in the environment and
hypothesis testing for providing performance bounds,
our off-policy setting tackles a different problem.
In off-policy evaluation, 
the model of the environment is generally unknown and 
we cannot deploy the decision policy in the environment 
because it could be costly or dangerous.
Unlike SMC, we cannot directly collect statistics with the decision policy.
Conversely, we try to approximate the expected performance 
using data collected by another policy.

\subsubsection{RL with temporal logic}
Much prior work adopts \emph{temporal logic} (TL) in RL.
Some of it focuses on the decomposition of a complex task into many sub tasks~\cite{DBLP:journals/corr/Jothimurugan2021_CompositionalRLFromLogicalSpecs,icarte2018ltl-tasks}. Other on
formulations tailored to tasks specified in TL~\cite{DBLP:conf/rss/FuTopcu2014_PACLearningWithTLConstraints,li2018tl-policy-search,Jiang2021TLRewardShaping,DBLP:conf/icml/Icarte2018_reward_machines}.
Several works use the quantitative semantics of TL (i.e., STL and its variants) to derive a reward signal ~\cite{li2017tl-rewards,arxiv/Jones2015RobustSatOfTLSpecViaRL,IROS/Anand2019StructuredRewardShapingUsingSTL}.
However, they describe the task as a monolithic TL specification and 
compute the reward on complete ~\cite{li2017tl-rewards} or truncated trajectories~\cite{IROS/Anand2019StructuredRewardShapingUsingSTL}.
In this work, we use HPRS, introduced in  \cite{berducci2021_hprs} to mitigate the reward sparsity and subsequent credit-assignment problem, which combines the individual evaluation of requirements
into a hierarchically-weighted reward, capturing
the priority among the various classes of requirements.

\subsubsection{Multi-objective RL.} 
Multi-objective RL (MORL) studies the optimization of multiple and often conflicting objectives.
MORL algorithms learn single or multiple policies ~\cite{roijers2013mo-sequential,liu2015morl}. 
There exist several techniques to combine multiple reward signals into a single scalar value (i.e., scalarization), such as linear or non-linear projections ~\cite{natarajan2005multi-criteria,barrett2008multiple-criteria,vanMoffaert2013morl}.
Other approaches formulate structured rewards by imposing or assuming a preference ranking on the objectives and finding an equilibrium among them \cite{icml/GaborKS98MultiCriteriaRL,NIPS/Shelton2000BalancingMultipleSourcesOfRewardInRL,yun2010ranking,PMLRr/Abels19DynamicWeightsInMODRL}.
\cite{DBLP:conf/ijcnn/Brys2014_MultiObjectivizationOfRLProblems}
proposes to decompose the task specification into many requirements.
However, they do not consider any structured priority among requirements
and rely on the arbitrary choice of weights for each of them.
In general, balancing between safety and performance shows connections with MORL. However, we focus on the safety aspects and how to guarantee safety below a certain threshold with high probability.
These characteristics are not present in MORL approaches.

\subsubsection{Hierarchically Structured Requirements.}
The use of partially ordered requirements to formalize complex tasks
has been proposed in a few works.
The \textit{rulebook} formalism ~\cite{DBLP:conf/icra/CensiSWYPFF19} represents a set of prioritized requirements, and it has been used for evaluating the behaviors produced with a planner ~\cite{DBLP:conf/icra/CensiSWYPFF19}, 
or generating adversarial testing scenarios ~\cite{DBLP:conf/rv/ViswanadhaRV2021_MultiObjectiveFalsificationScenicVerifAI}.
In~\cite{journals/ral/Puranic2021LerarningFromDemonstrationUsingSTLInStochAndContDomains}, a complementary inverse RL approach proposes learning dependencies for formal requirements starting from existing demonstrations. 
Conversely, we use the requirements to describe our task and
define a Constrained MDP, and then focus on the design of a safe policy-improvement algorithm.

\subsubsection{Model-based reinforcement learning}
Among the first representatives of model-based RL algorithms is PILCO \cite{deisenroth2011pilco}, 
which learned a Gaussian process for low-dimensional state systems. 
More recent works showed that it is possible to exploit
expressive neural-networks models to learn complex dynamics in robotics systems \cite{nagabandi2018neural}, and use them for planning \cite{chua2018pets} or policy learning \cite{janner2019mbpo}.
In the context of control from pixel images, 
world models \cite{ha2018world} proved that it is possible to learn 
accurate dynamic models for POMDPs by using noisy high-dimensional observations instead of accurate states.
Their application in planning \cite{hafner2019_planet}, and later in policy learning \cite{hafner2019_dreamer}, have achieved the new state-of-the-art performance in many benchmarks and were recently applied to real-world robots \cite{brunnbauer2021_racing_dreamer}.

\section{Preliminaries}

\subsection{Reinforcement Learning}
Reinforcement Learning (RL) aims to infer an intelligent agent's policy that takes actions in an environment in a way that maximizes some notion of expected return. 
The environment is typically modeled as a Markov decision process (MDP), and the return is defined as the cumulative discounted reward.

\begin{definition}
A Markov Decision Process (MDP) is a tuple
$M = (S, A, p, R, \gamma)$, where $S$ is a set of
states; $A$ is a set of possible actions;
$p : S \times A \times S \rightarrow [0, 1]$ is a transition probability function
(where $p(s_{t+1}|a_t, s_t)$ describes the probability of arriving in state $s_{t+1}$ if
action $a_t$ was taken at state $s_t$); 
$R : S \times A \times S \rightarrow \RR $ is a deterministic reward
function, assigning a scalar value to a transition;
$\gamma \in [0,1]$ is the discount factor that balances the importance of achieving future rewards.
\end{definition}

In RL, one aims to find a policy $\pi : S \times A \rightarrow [0, 1]$ 
which maps states to action probabilities, such that it maximizes
the expected sum of rewards collected over episodes (\textit{or trajectories}) $\tau$:
$$
\pi^{\star} = \argmax_{\pi} \EE_{\tau \sim \mu(\cdot|\pi)} \big[ \sum_{t=0}^{\infty} \gamma^t R(s_t,a_t,s_{t+1}) \big] ,
$$
where $\mu(\tau|\pi)$ represents the distribution 
over episodes observed when sampling actions from some policy $\pi$,
and $\tau \sim \mu(\cdot|\pi)$ denotes an episode that was sampled from this distribution.

Conventional RL approaches do not explicitly consider safety constraints in MDPs. Constrained MDPs \cite{Altman1999_CMDP} extend the MDP formalism
to handle such constraints, resulting in the tuple $M = (S, A, p, R, C, \gamma, d)$,
where $C : S \times A \times S \rightarrow \RR $ is a cost function and 
$d \in \RR $ is a cost threshold. 
We aim to teach the agent to \textit{safely} interact with the environment.
More concretely, 
considering the episodes $\tau \sim \mu(\cdot|\pi)$,
the expected cumulative cost must be below the threshold $d$; 
that is, the constraint is $\EE_{\tau \sim \mu(\cdot|\pi)} \big[ \sum_{t=0}^{\infty} \gamma^t C(s_t,a_t,s_{t+1},) \big] \leq d$.
Formally, the constrained optimization problem consists of finding $\pi^{\star}$ such that:
\begin{align*}
\pi^{\star} = 
&\argmax_{\pi} \EE_{\tau \sim \mu(\cdot|\pi)} \big[ \sum_{t=0}^{\infty} \gamma^t R(s_t,a_t,s_{t+1},) \big] \\
&\text{subject to } \EE_{\tau \sim \mu(\cdot|\pi)} \big[ \sum_{t=0}^{\infty} \gamma^t C(s_t,a_t,s_{t+1},) \big] \leq d
\end{align*}

We denote the expected cumulative discounted reward and cost as $V^{\pi}$ and $V_C^{\pi}$, respectively.
When dealing with multiple constraints, we use $C_i$, $d_i$, and $V_{C_i}^{\pi}$ to denote the $i$-th cost function, its threshold, and expected cumulative cost.

\subsection{Hierarchical Task Specifications}

\subsubsection{Requirements specification}

In \cite{berducci2021_hprs}, we formally define a set of expressive operators to capture requirements often occurring in control problems.
Considering atomic predicates $p \doteq f(s)\,{\geq}\,0$ over observable states $s\,{\in}\,S$, 
we extend existing task-specification languages 
(e.g., SpectRL \cite{DBLP:journals/corr/Jothimurugan2021_CompositionalRLFromLogicalSpecs})
and define requirements as: 
\begin{align}
\label{task:syntax}
\begin{split}
\varphi  \doteq  \ \texttt{achieve}\ p  ~|~ \texttt{conquer}\ p ~|~
                 \  \texttt{ensure}\ p  ~|~ \texttt{encourage}\ p 
\end{split}
\end{align}
Commonly, a task can be defined as a set of requirements from three basic classes: {\em safety}, {\em target}, and {\em comfort}. Safety requirements, of the form $\texttt{ensure}\ p$, are naturally associated to an invariant condition $p$. Target requirements, of the form  $\texttt{achieve}\ p$ or  $\texttt{conquer}\ p$, formalize the one-time or respectively the persistent achievement of a goal within an episode. Finally, comfort requirements, of the form $\texttt{encourage}~p$, introduce the soft satisfaction of $p$, as often as possible, without compromising task satisfaction.

Let $\Tau$ be the set of all finite episodes of length $T$. Then, each requirement $\varphi$ induces a Boolean function $\sigma:\,\Tau\,{\rightarrow}\,\BB$ evaluating whether an episode $\tau\,{\in}\,\Tau$ satisfies the requirement $\varphi$. Formally, given a finite episode $\tau$ of length $T$, 
we define the requirement-satisfaction function $\sigma$ as follows:
\begin{align*}
&\tau \models \texttt{achieve}\ p &    &\text{ iff } \exists i \leq T, \tau_i \models p \\
&\tau \models \texttt{conquer}\ p &    &\text{ iff } \exists i \leq T, \forall j \geq i, \tau_j \models p  \\
&\tau \models \texttt{ensure}\ p  &    &\text{ iff } \forall i \leq T, \tau_i \models p \\
&\tau \models \texttt{encourage}\ p &    &\text{ iff } \textsf{true}
\end{align*}

We explain below the proposed evaluation of comfort requirements.
In our interpretation, they represent secondary objectives, and their satisfaction does not alter the truth of the evaluation.
In fact, as long as the agent is able to safely achieve the target, we consider the task satisfied.
For this reason, the satisfaction of comfort requirement always evaluates to true.
To further clarify the proposed specification language, 
we formalize the requirements for the running example.

\begin{example}
\label{ex:task}
Consider the motivating cart-pole example. Now let us give the formal specification of its requirements. The state is the tuple $(x, \dot{x}, \theta, \dot{\theta})$, where $x$ is the position of the cart, $\dot{x}$ is its velocity, $\theta$ is the angle of the pole to the vertical axis, and $\dot{\theta}$ is its angular velocity. We first define: (1)~the angle $\theta_{\emph{max}}$ of the pole at which we consider the pole to fall from the cart; (2)~the maximum angle $\theta_{\emph{comf}}$ of the pole that we consider to be comfortable; (3)~the world limit $x_{\emph{lim}}$; (4)~the position $G$ of the goal; (5)~the set of points $O$ defining the static obstacle; and (6)~a distance function $d$ between locations in the world (e.g., euclidean distance), that, with a slight abuse of notation, we extend to measure the distance between the cart position and a set of points. 
Then, the task can be formalized with the requirements
reported in Table~\ref{tab:requirements}.

\begin{table}[t]
    \caption{Cart-pole example -- formalized requirements}
    \label{tab:requirements}
    \centering
    \begin{tabular}{|l|l|l|l|}
    \hline
    Req Id & Formula Id & Formula \\
    \hline \hline
    Req1 & $\varphi_1$ & $\texttt{conquer } d(x,G) = 0$ \\
    Req2 & $\varphi_2$ & $\texttt{ensure } |\theta| \leq \theta_{\emph{max}}$ \\
    Req3 & $\varphi_4$ & $\texttt{ensure } d(x,O) > 0$ \\
    Req4 & $\varphi_3$ & $\texttt{ensure } |x| \leq x_{\emph{lim}}$ \\
    Req5 & $\varphi_5$ & $\texttt{encourage } |\theta| \leq \theta_{\emph{comf}}$ \\
    \hline
    \end{tabular}
\vspace*{-3ex}
\end{table}
\end{example}

\subsubsection{Task as a Partially-Ordered Set}

We formalize a task by a set of formal requirements $\Phi$, assuming that the target is unique and unambiguous. Formally, $\Phi = \Phi_S \uplus \Phi_T \uplus \Phi_C$ such that:
$$
\begin{array}{lc}
    \Phi_S := \{ \varphi \, | \,  \varphi \doteq \texttt{ensure}\ p \} \\
    \Phi_C := \{ \varphi \, | \,  \varphi \doteq \texttt{encourage}\ p \} \\
    \Phi_T := \{ \varphi \, | \, \varphi \doteq \texttt{achieve}\ p \,\vee\, \varphi \doteq \texttt{conquer}\ p \} 
\end{array}
$$
\noindent The target requirement is required to be unique ($|\Phi_T| = 1$). 



We use a very natural interpretation of importance among the class of requirements, which considers decreasing importance from safety, to target, and to comfort requirements.
Formally, this natural interpretation of importance defines a (strict) partial order relation $\strictorder$ on $\Phi$ as follows:
$$
\varphi \strictorder \varphi' \text{ iff } \left( \varphi \in \Phi_S \wedge \varphi' \not \in \Phi_S\right) \vee \left( \varphi \in \Phi_T \wedge \varphi' \in \Phi_C \right)
$$

%

The resulting pair $(\Phi, \strictorder)$ forms a partially-ordered set of requirements and defines our task. Extending the semantics of satisfaction to a set, we consider a task accomplished when all of its requirements are satisfied:
\begin{equation}
\label{eq:task_sat}
\tau \models \Phi \text{ iff } \forall \varphi \in \Phi, \tau \models \varphi.    
\end{equation}

\section{Contribution}

In this section, we present the main contribution of this work: a correct-by-construction RL pipeline to solve formally-specified control tasks.
First, we formalize a CMDP from the set of requirements,
providing the intuition behind its sound formulation.
Then, we describe the potential-based reward proposed in \cite{berducci2021_hprs} 
that we use to enrich the learning signal and still benefit from
correctness guarantees.
Finally, we present an online RL algorithm that iteratively updates a policy while maintaining the performance for safety requirements.

\subsection{Problem Formulation}
\label{sec:problem_formulation}

The environment is considered a tuple $E\,{=}\,(S, A, p)$, 
where $S$ is the set of states, $A$ is the set of actions, 
and $p(s_{t+1}|s_t, a_t)$ is its dynamics, that is, the probability of reaching state $s'$ by performing action $a$ in state $s$.
Given a task specification $(\Phi, \strictorder)$, where $\Phi = \Phi_S \uplus \Phi_T \uplus \Phi_C$,
we define a CMDP $M\,{=}\,(S, A, p, R, C, \gamma, d)$ by 
formulating its reward $R$ and cost functions $C_i$ to reflect the semantics of $(\Phi, \strictorder)$.

We consider episodic tasks of length $T$, where the episode ends when the task satisfaction is decided: either through a safety violation, a timeout, or the goal achievement. 
When one of these events occurs, we assume that the MDP is entering a \textit{final} absorbing state $s_f$, where the decidability of the episode cannot be altered anymore.
The goal-achievement evaluation depends on the target operator adopted: 
for $\texttt{achieve}\ p$ the goal is achieved when visiting at time $t\,{\leq}\,T$ a state $s_t$ such that $s_t \models p$; 
for $\texttt{conquer}\ p$ the goal is achieved if there is a time $i\,{\leq}\,T$ such that for all $i\,{\leq}\,t\,{\leq}\,T$, $s_t \models p$.

We adopt a straightforward interface for safety requirements,
requiring the cost function $C_i(s, a, s')$ to be a binary indicator of the violation of the $i$-th safety requirement $\phi_i \in \Phi_S$ when entering $s'$ from $s$: 
0 if the current state satisfies $\phi_i$ and $1$ if the current state violates $\phi_i$.
We bound the expected cumulative discounted cost 
by $d_i$, a user-provided safety threshold that depends on the specific application considered:

$$V_{C_i}^\pi = \EE_{\tau \sim \mu(\cdot|\pi)} [ \sum_{t=0}^{T} \gamma^t C_i(s_t, a_t, s_{t+1}) ] \leq d_i $$ 

The choice of this cost function promotes simplicity,
requiring users to only be able to detect safety violations and relieving them from the burden of defining more complex signals.
Moreover, the resulting cost metric $V_{C_i}^\pi$ reflects the failure probability discounted over time by the factor $\gamma$ and makes the choice of the threshold $d_i$ more intuitive by interpreting it in a probabilistic way.

We complete the CMDP formulation by defining a sparse reward, incentivizing goal achievement. 
Let $p$ be the property of the unique target requirement. 
Then, for safe transitions $\langle s_t, a_t, s_{t+1} \rangle$, we define the following reward signal:
\begin{align*}
& R(s_t, a_t, s_{t+1}) =
\left\{ 
  \begin{array}{ c l }
    1           & \quad \textrm{if } s_{t+1} \models p \\
    0           & \quad \textrm{otherwise}
  \end{array}
\right.
\end{align*}


The rationale behind this choice is that the task's satisfaction depends 
on the satisfaction of safety and target requirements.
In the same way, the reward $R$ incentives safe transitions to states that satisfy the target requirement, and the violation of any safety requirements terminates the episode, precluding the agent from collecting any further reward.
It follows that $R$ incentives to reach the target and stay there as long as possible, in the limit until $T$.

\subsection{Reward Shaping}
\label{reward_shaping}
We additionally define a shaped reward with HPRS \cite{berducci2021_hprs}
to provide a dense training signal and speed-up the
learning process.
Since we consider a CMDP and model the safety requirements as constraints, we restrict the HPRS definition to target and comfort requirements.

We consider predicates $p(s)\,{\doteq}\,f(s)\,{\geq}\,0$, 
where the value $f(s)$ is bounded in $[m, M]$ for all states $s$. 
Let $f_{-}(s)\,{=}\,min(0,f(s))$ and $\Phi_{tc} = \Phi_T \uplus \Phi_C$.
We use the continuous normalized signal $r(\varphi, s) = 1 - \frac{f_{-}(s)}{m}$
to define the potential shaping function $\Psi$ as follows:
$$
\Psi(s) = \sum_{\varphi \in \Phi_{tc}}  \left( \prod_{\varphi' \in \Phi_{tc} : \varphi' \strictorder \varphi} r(\varphi', s) \right) \cdot r(\varphi, s)
$$

This potential function is a weighted sum over all scores $r(\varphi,s)$ for target and comfort requirements, where the weights are determined by a product of the scores of all specifications that are strictly more important (hierarchically) than $\varphi$. Crucially, these weights adapt dynamically at every step.

\begin{corollary}
The optimal policy for the CMDP $M'$, where its reward $R'$ is defined as below:
\begin{equation}
    \label{eq:hrs}
    R'(s_t, a_t, s_{t+1}) = R(s_{t}, a_{t}, s_{t+1}) + \gamma \Psi(s_{t+1}) - \Psi(s_{t})
\end{equation}
is also an optimal policy for the CMDP $M$ with reward $R$.
\end{corollary}

The corollary stated in \cite{berducci2021_hprs} follows
by the fact that $\Psi$ is a potential function (i.e.,  depends only on the current state) and by the results in \cite{ICML/Ng19999PolicyInvarianceUnderRewardTransformation}.
This result remark that the proposed reward shaping is correct
since it preserves the policy optimality of the CMDP $M$.

The proposed hierarchical-potential signal has a few crucial characteristics.
First, it is a potential function and can be used 
to augment the original reward signal 
without altering the optimal policy of the resulting CMDP.
Second, it is a multivariate signal that combines target and 
comfort objectives with multiplicative terms.
A linear combination of them, as typical in multi-objective scalarization,
would assume independence among objectives.
Consequently, any linear combination would not be expressive enough
to capture the interdependence between requirements.
Finally, the weights dynamically adapt at every step according to the satisfaction degree of the requirements.

\subsection{Safe Policy Improvement in Online Setting}

This section presents an online RL algorithm that uses a correct-by-construction policy-improvement routine.
While this approach is general enough to be used on any CMDP, 
we use it with the shaped reward signal and costs presented in the previous section.

At each iteration, the algorithm performs a correct-by-construction refinement of the current policy $\pi$, with high probability. 
This means that, with high probability, the algorithm returns a policy $\pi'$ whose safety performance is not worse than that of $\pi$.
Since we are working in a model-free setting, with access to only a finite
amount of off-policy data, we can define correct-by-construction only up to certain confidence $\delta$.
Below is the formal definition.

\begin{definition}
Let $M$ be a CMDP with cost functions $C_i$ for $i \in [1,k]$,
$\pi$ a policy, 
and $D_n \sim \mu(\cdot | \pi)$ a finite dataset of experiences collected with $\pi$. A policy-improvement routine $\mathcal{A}$ is correct-by-construction if
for any $\delta$:

\begin{align*}
\pi' = \mathcal{A}(\pi, D_n, \delta) \quad \text{ s.t. } \quad   
Pr( \bigwedge_{i=1}^k V_{C_i}^{\pi'} \leq V_{C_i}^{\pi} ) \geq 1 - \delta
\end{align*}

\end{definition}

Having defined what correct-by-construction means, we describe how we can build an update mechanism to prevent the deployment of an unsafe policy.

\subsubsection{High-confidence off-policy policy evaluation}
Before releasing a policy $\pi'$ for deployment in the environment,
we need to estimate its safety performance using a set of trajectories $D_{n}$ 
collected with the previous deployed policy $\pi$.
We assume to know a threshold $\rho_{+}$, being either an acceptable upper bound for the cost in our application or an estimate of the performance of the last deployed policy $\pi$.
We aim to evaluate a candidate policy $\pi'$ and check if its expected cumulative costs are below the thresholds with a probability of at least $1\,{-}\,\delta$.

We use \textit{importance sampling} to produce an unbiased estimator of $\rho(\pi')$ over a trajectory $\tau$ collected by running $\pi$.
The estimator is defined as $$\hat{\rho}(\pi' | \tau, \pi) = 
\prod_{t=1}^{|\tau|} \frac{\pi'(a_t|s_t)}{\pi(a_t|s_t)}
\big( \sum_{t=1}^{|\tau|} \gamma^t C(s_t,a_t,s_{t+1}) \big)
$$

Computing the \textit{importance weighted returns} gives us unbiased estimators of the safety performance \cite{precup2000_ope}.
The mean over estimators from $n$ trajectories in $D_{n}$ is also unbiased, $\hat{\rho} = \frac{1}{n} \sum_{i}^{n} \hat{\rho}(\pi' | \tau_i, \pi)$.
However, we want to provide statistical guarantees regarding the resulting value.
We use the one-sided Student-t test to obtain a $1-\delta$ confidence upper bound on $\hat{\rho}$.

Let $\hat{\rho}_i$ be the $i$-th unbiased estimator obtained by importance sampling,
the Student-t test defines:
$$\hat{\rho} = \frac{1}{n} \sum_{i}^{n} \hat{\rho}_i,$$
$$\sigma = \sqrt{ \frac{1}{n - 1} \sum_{i}^{n} (\hat{\rho}_i - \hat{\rho} )^2 }$$
and proves that with probability $1-\delta$ that:
$$\rho \leq \hat{\rho} + \frac{\sigma}{\sqrt{m}} t_{1-\delta, n-1} $$
where $t_{1-\delta, n-1}$ is the $(1-\delta)100$ percentile of the Student's t distribution 
with $n-1$ degrees of freedom.
Under the assumption of normally distributed $\hat{\rho}$, 
which is a reasonable assumption \cite{thomas2015_hcpe} for 
large $n$ by \textit{central limit theorem} (CLT),
we use the Student's t-test to obtain a guaranteed upper bound on $\rho$.
If the upper bound is below the threshold $\rho_{+}$,
we can release the current policy $\pi$ in the environment
because we know that its expected cumulative cost is not higher than $\rho_{+}$ 
with probability at least $1-\delta$.

\subsubsection{Safe model-free policy improvement (SMFPI)}
Among the most successful approaches in model-free optimization
are \textit{policy gradient} methods \cite{schulman2017_ppo}.
They update the policy by estimating the policy gradient 
over a finite batch of episodes $D_n$.
A typical gradient estimator for a policy $\pi_\theta$ parametrized by $\theta$
is
$$
\hat{g} = \EE_{(s_t, a_t) \in D_n} 
\big[ \nabla_\theta \text{ log } \pi_\theta(a_t | s_t) \hat{A}_t \big]
$$
where $\hat{A}_t$ is an estimator of the advantage function and, in its simplest form, corresponds to the discounted cumulative return.

The main limitation of these approaches is due to the on-policy nature of policy gradient methods.
At each policy update, they must collect new data by interacting with the environment to compute the gradients.
Reusing the same trajectories to perform many updates is not theoretically justified and can perform catastrophic policy updates in practice.
The first algorithm we propose, SMFPI,
directly uses this policy-gradient update and estimates the return as the sum of the shaped rewards, as presented in the previous section.
The following subsection discusses a model-based solution to improve data efficiency by learning a dynamics model.

\subsubsection{Safe model-based policy improvement (SMBPI)}
Model-based algorithms are known to be more sample efficient than model-free ones. 
The data collected over the training process can be used to fit a dynamics model that serves as a simulator of the real environment.

Despite the variety of model-based approaches, we use the dynamics model for training a model-free agent on predicted trajectories, reducing the required interactions with the real environment.
The dynamics model is reusable over many iterations, and for this reason, we consider it a data-efficient alternative.

\paragraph{Learning the dynamics.}

We consider the problem of learning an accurate dynamical model.
Traditional approaches use Bayesian models (e.g., GPs) \cite{deisenroth2011pilco} for their efficiency in low-data regimes,
but training on a large dataset and high-dimensional data is prohibitive.
Modern literature in deep RL \cite{chua2018pets} suggests that using an ensemble of neural networks 
can produce competitive performance with scarce data and efficiently scale to large-data regimes.

Since we intend to represent a potentially stochastic state-transition function, capturing the noise of observations and process,
we train a model to parametrize a probability distribution.
We assume a diagonal Gaussian distribution model $p_\theta$
and let a neural network predict the mean $\mu_{\theta}(s, a)$ and the
log standard deviation $log\,\sigma_{\theta}(s, a)$.
Instead of learning to predict the following state $s_{t+1} = p_\theta(s_t, a_t)$,
we train our dynamics model to predicts the change with respect to the current state $\Delta s_{t+1} = p_\theta(s_t, a_t)$ \cite{deisenroth2013gaussian,nagabandi2018neural}.

However, training on scarce data may lead to overfitting and extrapolation errors in the area of the state-action space that are not sufficiently supported by the collected experience.
A common solution to prevent the algorithm from exploiting these regions
consists of adopting an ensemble of models $\{p_{\theta_i}\}_{i=1}^{m}$.
With this ensemble representing a finite set of plausible dynamics of our system, we predict the state trajectories and propagate the dynamics uncertainty by shooting $N$ particles from the current state $s_t$.
At each timestep, we uniformly sample one of the ensemble's models.
The choice of the ensemble size (i.e., the number $m$ of dynamics models) is crucial to capture the uncertainty in the underlying stochastic dynamics adequately. 
While a small ensemble might introduce a significant model bias in policy optimization, a large ensemble increases the computational and memory cost for learning and storing the models. 

\paragraph{Policy optimization with learned dynamics.}
The dynamical model defines an MDP that approximates the real environment. This provides a simulator 
from which we can sample plausible trajectories
without harmful interaction.

Starting from true states sampled from
our buffer of past experiences,
we generate predictions using the dynamical
model 
$\hat{s}_{t+1} \sim p_\theta(\hat{s}_t, a_t)$,
for action $a_t \sim \pi(\hat{s}_{t})$.
We assume the reward function $r_t = R(\hat{s}_{t}, a_t, \hat{s}_{t+1})$ and 
the cost functions 
$c_{i,t} = C_i(\hat{s}_{t}, a_t, \hat{s}_{t+1})$
to be known.
In general, the learned model could predict them, so in the following, we refer to the predictive model
as able to generate state, reward and costs, i.e., $(s_{t+1}, r_t, c_{i,t}) \sim p_\theta(s_t, a_t)$.

Considering the predicted trajectories over a finite horizon $H$, we use model-free RL to train the policy.
This approach is agnostic to the specific algorithm adopted. 
We consider two approaches for dealing with safety:

\begin{enumerate}
    \item \textit{Pessimistic Reward.} We alter the MDP transitions when visiting unsafe states (where the cost exceeds the threshold). In this case, the system enters an absorbing state where it repeatedly collects a high penalty $-C$.
    \item \textit{Constrained optimization.} We define the unconstrained Lagrangian objective
    $$V(\pi) - \sum_{i=1}^{C} \lambda_i V_{C_i}(\pi)$$
    for some multiplier $\lambda_i$.
    Using gradient-based optimization, we iteratively update the policy parameters and the dual variables $\lambda_i$.
\end{enumerate}

\subsubsection{Correct-by-construction RL algorithms.}

Having introduced various algorithmic solutions to tackle the policy optimization process,
we report the pseudocode of the complete algorithm in Algorithm \ref{algo}.
Then, we state an important result derived from adopting high-confidence off-policy evaluation.

\begin{theorem}
Algorithm~\ref{mfree_policy_improvement} (SMFPI) and Algorithm~\ref{mbased_policy_improvement} (SMBPI) are correct-by-construction policy-improvement routines.
\end{theorem}

\begin{proof}
Using the results from high-confidence off-policy evaluation literature \cite{thomas2015safe},
we demonstrate that each policy $\pi'$ returned by SMFPI or SMBPI
has expected cost less than or equal to the expected cost of the initial policy $\pi$
with probability at least $1 - \delta$.
If $\pi'=\pi$, then the condition is trivially satisfied, so 
let us assume $\pi'$ to be a different policy.
According to the algorithm, $\pi'$ satisfies:
$$
\rho^i \leq \rho_{+}^i ,~ \forall i \in [1,k]
$$
By definition of $\rho^i$ using Student-t test and confidence $\frac{\delta}{k}$,
and under the assumption of normally-distributed cost sample means,
we know that for each $i$:
$$
Pr \big( V_{C_i}^{\pi'} \leq V_{C_i}^{\pi} \big) \geq 1 - \frac{\delta}{k}
$$
and equivalently that
$$
Pr \big( V_{C_i}^{\pi'} > V_{C_i}^{\pi} \big) < \frac{\delta}{k}
$$
Using the Union Bound, we finally show that the probability of the event 
in which at least one of the costs violates the correctness condition is at most $\delta$.
$$
Pr \big( \bigvee_{i=1}^k V_{C_i}^{\pi'} > V_{C_i}^{\pi} \big) < 
\sum_{i=1}^k Pr \big( V_{C_i}^{\pi'} > V_{C_i}^{\pi} \big) = \delta
$$
This statement is equivalent to the following condition
which concludes the proof.
$$
Pr \big( \bigwedge_{i=1}^k V_{C_i}^{\pi'} \leq V_{C_i}^{\pi} \big) \geq 1 - \delta.
$$
\end{proof}

\begin{algorithm} 
\caption{Safe Policy Optimization} 
\label{algo} 
\begin{algorithmic} 
    \REQUIRE Initial policy $\pi$, confidence $\delta$ for policy improvement
    \STATE Initialize prediction model $p_\theta$, empty buffer $D$ and $D_{test}$
    \FOR{$N$ epochs}
        \STATE Collect data with $\pi$ in environment: $$D = D \cup \{ (s_t, a_t, s_{t+1}, r_t, c_t^1, ..., c_t^k, \text{log } \pi(a_t|s_t) ) \}_t$$
        \STATE Estimate safety performance of current policy $\rho_{+}^i(\pi) ,~ \forall i \in [1,k]$
        \STATE Optimize policy with SPI (Algorithm \ref{mfree_policy_improvement} or \ref{mbased_policy_improvement}): 
                $$\pi = \texttt{SPI}(\pi, D, \rho_{+}, \delta)$$
    \ENDFOR
    \RETURN $\pi$
\end{algorithmic}
\end{algorithm}

\begin{algorithm} 
\caption{Safe Model-free Policy Improvement (SMFPI)} 
\label{mfree_policy_improvement} 
\begin{algorithmic} 
    \REQUIRE    Initial policy $\pi$, 
                data $D \sim \mu_{(\cdot | \pi)}$, 
                safety thresholds $\rho^i_{+}$, confidence $\delta$
    \STATE $\pi' = \pi$
    \STATE Split $D$ into $D_{train}$ and $D_{test}$.
    \FOR{$L$ epochs}
        \STATE Perform $g$ policy updates of policy $\pi'$ with data $D_{train}$
        \STATE Upper bound costs: $\forall i=1..k ~,~ \rho^i(\pi') = HCOPE(\pi', D_{test}, \delta / k)$
        \IF{$\forall i \in [1,k] \, \rho^i \leq \rho_{+}^i$}
            \RETURN $\pi'$
        \ENDIF
    \ENDFOR
    \RETURN $\pi$ \hfill\COMMENT{no better policy found}
\end{algorithmic}
\end{algorithm}

\begin{algorithm} 
\caption{Safe Model-based Policy Improvement (SMBPI)} 
\label{mbased_policy_improvement} 
\begin{algorithmic} 
    \REQUIRE    Initial policy $\pi$, 
                data $D \sim \mu_{(\cdot | \pi)}$, 
                safety thresholds $\rho^i_{+}$, confidence $\delta$
    \STATE Initialize or reuse predictive model $p_\theta$
    \STATE Empty model buffer $D_{model}$
    \STATE $\pi' = \pi$
    \STATE Split $D$ into $D_{train}$ and $D_{test}$.
    \FOR{$L$ epochs}
        \STATE Fit dynamics model $p_\theta$ on dataset $D_{train}$ 
        $$\theta \leftarrow \argmax_{\theta} \EE [ \text{log} \, p_\theta(s', r, c | s, a) ]$$
        \STATE Collect $h$-step predicted trajectories with $\pi'$ and store in $D_{model}$
        \STATE Perform $g$ policy updates of policy $\pi'$ with data $D_{model}$
        \STATE Upper bound costs: $\forall i=1..k ~,~ \rho^i(\pi') = HCOPE(\pi', D_{test}, \delta / k)$
        \IF{$\forall i=1..k \, \rho^i \leq \rho_{+}^i$}
            \RETURN $\pi'$
        \ENDIF
    \ENDFOR
    \RETURN $\pi$ \hfill\COMMENT{no better policy found}
\end{algorithmic}
\end{algorithm}


\section{Experiments}
In this section, we provide an empirical evaluation of the proposed approach. Since it relies on two distinct contributions, the shaping reward and the correct-by-construction policy improvement, we structure the experiments in two phases to evaluate each component in isolation.

\subsection{Automatic Reward Shaping from Task Specification}
We evaluate the proposed reward shaping against two standard RL benchmarks:
(1)~\emph{The cart-pole} environment, which has already been presented as a motivating example; (2)~\emph{The bipedal-walker}, whose main objective is to move forward towards the end of the field without falling. We formulate one safety, one target, and four comfort requirements for the bipedal walker.

Aiming to evaluate the automatic reward shaping process from formal task specifications, we benchmark HPRS against two prior approaches using temporal logic.
To formalize the task in a single specification, we consider it as the conjunction of safety and target requirements.
The baselines methods to shape a reward signal are the following:
\begin{itemize}
    \item \emph{TLTL}~\cite{li2017tl-rewards} specifies tasks in a bounded (Truncated) LTL variant, equipped with an infinity-norm quantitative semantics \cite{nickovic2004stl}.
    The quantitative evaluation of the episode is used as a reward. We employ the RTAMT monitoring tool to compute the episode robustness \cite{DBLP:conf/atva/Nickovic2020_RTAMT}.
    
    \item \emph{BHNR}~\cite{IROS/Anand2019StructuredRewardShapingUsingSTL} specifies tasks in a fragment of STL with the filtering semantics of \cite{DBLP:conf/hybrid/Rodionova2016_TemporalLogicAsFiltering}. The reward uses a sliding-window approach to produce more frequent feedback to the agent: at each step, it uses the quantitative semantics to evaluate a sequence of $H$ states.
\end{itemize}

Since each reward formulation has its scale, 
comparing the learning curves needs an external, unbiased assessment metric. 
To this end, we introduce a \textit{policy-assessment metric} (PAM) $F$, 
capturing the logical satisfaction of various requirements. 
We use the PAM to monitor the learning process and 
compare HPRS to the baseline approaches (Fig. \ref{fig:hprs}).

Let $\Phi = \Phi_{S} \uplus \Phi_{T} \uplus \Phi_{C}$ be the set of requirements defining the task. 
Then, we define $F$ as follows:
$$
F(\Phi, \tau) =  
\sigma(\Phi_S, \tau) + 
\tfrac{1}{2} \sigma(\Phi_T, \tau) + 
\tfrac{1}{4}\sigma_{avg}(\Phi_C, \tau)
$$

\noindent where $\sigma(\Phi, \tau)\,{\in}\,\{0,1\}$ is the satisfaction function evaluated over $\Phi$ and $\tau$. We also define a time-averaged version for any comfort requirement $\varphi = \texttt{encourage } p$, as follows:
$$
\sigma_{avg}(\varphi, \tau) = \sum_{i=1}^{|\tau|} \frac{\ONE(s_i \models p)}{|\tau|} .
$$

\noindent{}Its set-wise extension computes the average over the set.

\begin{corollary}
\label{pam_corollary}
Consider a task $(\Phi, \strictorder)$ and an episode $\tau$. Then, the following relations hold for $F$:\vspace*{-1mm}
\[\begin{array}{l}
F(\Phi, \tau) \geq 1.0 \leftrightarrow \tau \models \Phi_S,\quad 
F(\Phi, \tau) \geq 1.5 \leftrightarrow \tau \models \Phi
\end{array}\]
\end{corollary}

\begin{figure}
    \centering
    \includegraphics[width=\textwidth]{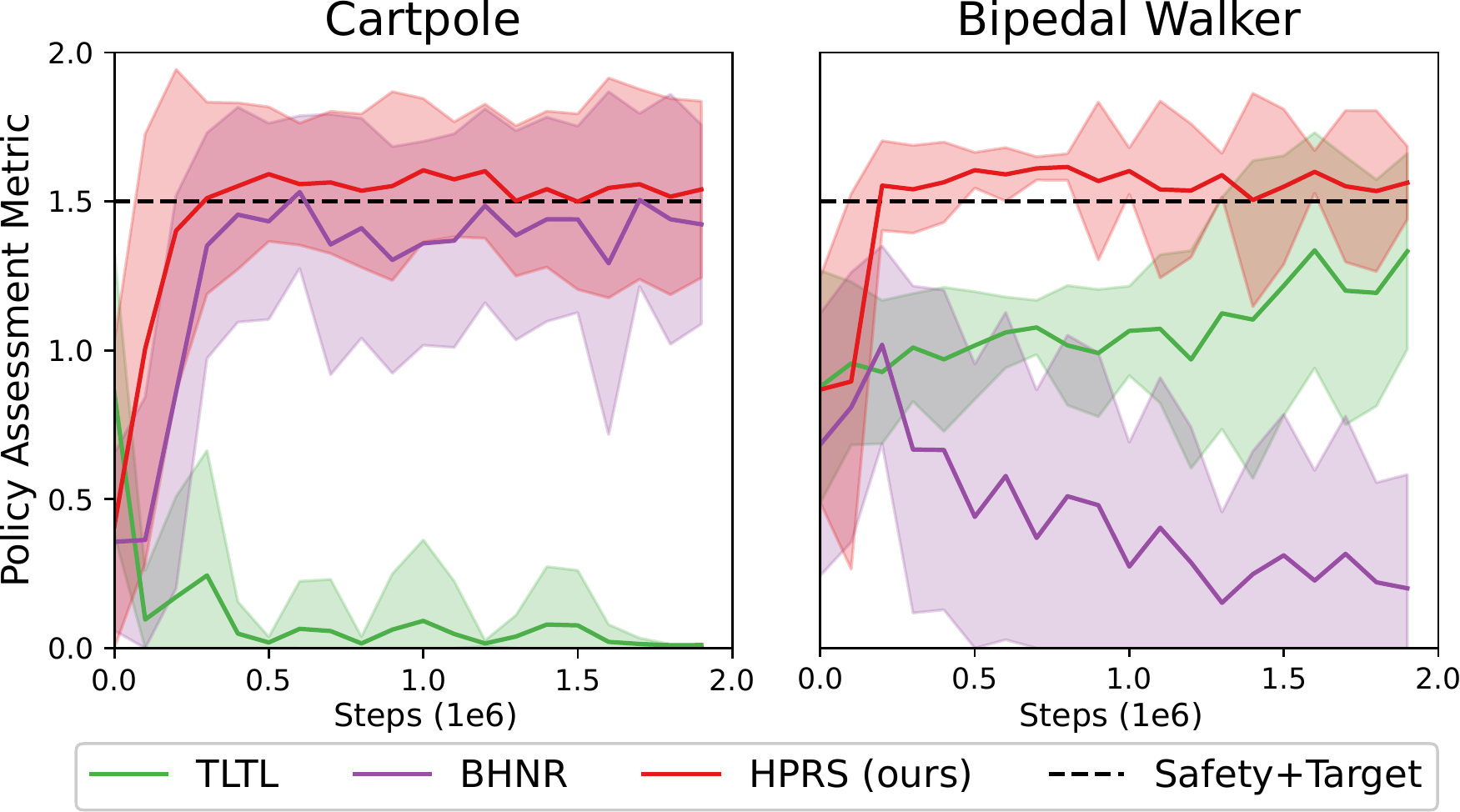}
    \caption{Performance evaluation with respect to the Policy Assessment Metric (PAM). This external metric is not used in training but to provide a sound evaluation of the requirements and accounting for the different rewards scales. The threshold indicates the task satisfaction (dashed line), as explained in Corollary \ref{pam_corollary}. All the curves report mean and standard deviation over five seeds.}
    \label{fig:hprs}
\end{figure}

\subsection{Safe Policy Improvement in Online Setting}

We evaluate the Safe Policy Improvement on a simple cart-pole task, 
where the agent target is to remain close to the goal location 
while keeping the pole safely balanced for an episode of $200$ steps.
We benchmark the proposed algorithm SMFPI with Vanilla Policy Gradient (VPG), which performs unconstrained updates based on the gradient estimates.
We consider two scenarios of hyperparameter configurations,
respectively with a favorable and unfavorable value of the \textit{learning rate} ($LR$).
Figure \ref{fig:exp_hcope} shows the performance as the mean and standard deviation of expected return and cost, aggregated over many runs.
We also report a running estimate of the expected cost for each run
to compare the oscillation of cost during the iterative policy updates.

\begin{figure}
    \centering
    \includegraphics[width=\textwidth]{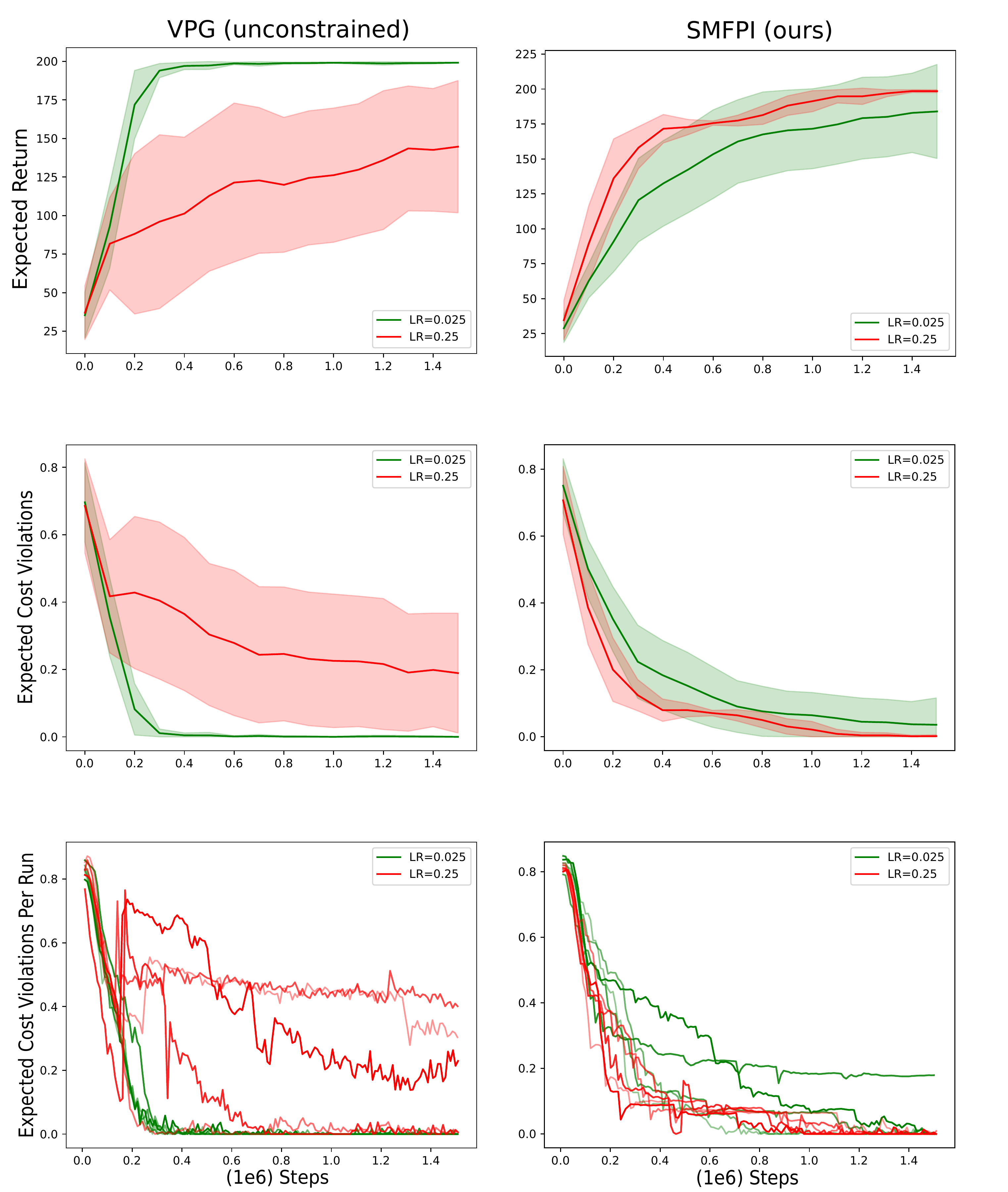}
    \caption{Performance for cumulative reward and cost violations under different hyperparameter configurations with learning rate ($LR$) $0.025$ and $0.25$. The aggregated performance in the first two rows indicates the mean and standard deviation over five runs. The last row shows the running cost estimate on the current batch for each of the five runs.}
    \label{fig:exp_hcope}
\end{figure}

Under a favorable hyperparameter setting $LR\,{=}\,0.025$, 
we can observe that both algorithms iteratively 
increase the expected return and decrease the expected cost.
As expected, VPG converges faster to the optimal value.
However, under the unfavorable hyperparameter setting $LR\,{=}\,0.25$,
the unconstrained policy update of VPG leads to high variance in the statistics and catastrophic oscillation in the cost per run.
Conversely, SMFPI shows a good policy improvement and much lower variance in the learning curve due to a more conservative policy update.
We also observe a monotonic decrease in expected cost,
empirically demonstrating the effectiveness of the proposed approach.
We observe small statistical fluctuations in the estimated cost per run,
which could depend on two factors: (1)~Online estimation of the cost on the current batch of data, (2)~Use of $\delta\,{=}\,0.05$, which guarantees a safe update up to certain confidence.
When the current batch contains few samples, 
the variance of the visualized cost estimate is high.
However, SMFPI does not update the policy in small-data regimes because it cannot guarantee the improvement with sufficient confidence.

In this experiment, we show that the proposed algorithm is 
effective for safely learning a decision policy.
Starting from a random policy, SMFPI can iteratively update 
the policy until it reaches the optimal return value
without showing any drop in expected cost.
Its performance is robust to different hyperparameters settings,
which plays a critical role in modern deep RL algorithms.
Compared with unconstrained RL algorithms under optimal hyperparameters, the convergence of SMFPI is slower. 
However, finding the correct hyperparameters by trial and error is not always an option in critical applications.


\section{Conclusion and Future Work}
In this work, we presented a policy-synthesis pipeline, showing how to formalize a CMDP, starting from a set of formal requirements describing the task.
We further enrich the reward with a potential function that does not alter
the policy optimality under the standard potential-based shaping.
We further define correct-by-construction policy-improvement routines and propose two online algorithms, one model-free and one model-based, to solve the CMDP starting from an unsafe policy.
Using high-confidence off-policy evaluation, we can guarantee that the proposed algorithms will return an equally good or improved policy with respect to the safety constraints.
We finally evaluate the overall pipeline, combining reward shaping and correct-by-construction policy improvement, and show empirical evidence of their effectiveness.
The proposed algorithm is robust under different hyperparameters tuning, while unconstrained baselines perform updates that deteriorate the safety performance and preclude the algorithm from converging to optimal performances.

Compared with previous approaches that mainly focus on a batch setting,
we propose to use high-confidence off-policy evaluation online.
In future work, we intend to investigate the proposed model-based approach regarding data efficiency and scalability of off-policy evaluation in an online setting.
We plan to extend the benchmarks with more complex tasks, targeting robotics applications and autonomous driving, study the convergence property more in-depth, and characterize the eventual distance to the optimal policy.

\section*{Acknowledgement}
Luigi Berducci is supported by the Doctoral College Resilient Embedded Systems.
This work has received funding from the Austrian FFG-ICT project ADEX.

%
%

\bibliographystyle{splncs04}
\bibliography{refs.bib}

\end{document}